%% file: k_learn_games.tex
\documentclass[accepted]{uai2021} %
\usepackage{natbib} %
\input{bod_macros}

\usepackage{url}  %
\usepackage{graphicx}  %
\usepackage{hyperref}
\usepackage{subcaption}

\usepackage{mathtools}
\usepackage{amsfonts}       %
\usepackage{amsthm}
\usepackage{algorithm}
\usepackage{algorithmic}

\newcommand{\Ac}{\mathcal{A}}

\newcommand{\Lc}{\mathcal{L}}
\newcommand{\Fc}{\mathcal{F}}

\newcommand{\Rc}{\mathcal{R}}
\newcommand{\Nc}{\mathcal{N}}
\newcommand{\Nat}{\mathbb{N}}
\newcommand{\lse}[1]{\log \sum_{j=1}^#1 \exp}
\newcommand{\xdim}{k}
\newcommand{\ydim}{m}

\newtheorem*{theorem*}{Theorem}
\newtheorem{theorem}{Theorem}
\newtheorem{lemma}{Lemma}

\newtheorem{assumption}{Assumption}

\title{Matrix games with bandit feedback}

\author[1]{Brendan O'Donoghue} %
\author[1]{Tor Lattimore}
\author[1]{Ian Osband}
\affil[1]{%
    DeepMind
}

\begin{document}
\maketitle

\begin{abstract}

We study a version of the classical zero-sum matrix game with unknown payoff
matrix and bandit feedback, where the players only observe each others actions
and a noisy payoff.  This generalizes the usual matrix game, where the payoff
matrix is known to the players.
Despite numerous applications, this problem has received relatively little
attention. Although adversarial bandit algorithms achieve low regret, they do
not exploit the matrix structure and perform poorly relative to the new
algorithms.
The main contributions are regret analyses of variants of UCB and K-learning
that hold for any opponent, \eg, even when the opponent adversarially plays the
best-response to the learner's mixed strategy.  Along the way, we show that
Thompson fails catastrophically in this setting and provide empirical
comparison to existing algorithms.

\end{abstract}

\section{Two-player zero-sum games}
\label{sec:problem}

Any two-player zero-sum game can be described by a payoff matrix $A \in
\reals^{m \times \xdim}$ \citep{neumann1928theorie, von2007theory}.
The row player selects $i \in \{1..\ydim\}$ and column player selects $j
\in \{1..\xdim\}$.
These choices are revealed simultaneously and the row player makes a payment of $A_{ij}$ to the column player.
In general, the optimal strategy for each player is \emph{mixed}, \ie\ determined by a probability distribution across actions.
We can therefore determine the optimal strategy for each player to maximize their reward:
\begin{eqnarray}
\label{eq:row}
\text{(row)} & \argmin_{y \in \Delta_\ydim} \max_i (A^T y)_i \\
\label{eq:column}
\text{(column)} & \argmax_{x \in \Delta_\xdim} \min_j (A x)_j,
\end{eqnarray}
where $\Delta_p$ is the probability simplex of dimension $p-1$.

The linear programs (LPs) \eqref{eq:row} and \eqref{eq:column} are dual, and
strong duality for LPs means that the optimal values for each problem are
identical \citep{boyd2004convex}.
We refer to this shared optimal quantity as the \emph{value} of the game, denoted $V_A^\star$.
\begin{equation}
\label{e-vonneumann}
  V_A^\star := \min_{y \in \Delta_\ydim} \max_{x \in \Delta_\xdim} y^T A x =
   \max_{x \in \Delta_\xdim} \min_{y \in \Delta_\ydim} y^T A x.
\end{equation}
Any primal-dual strategies $(x^\star, y^\star)$ that solve the saddle-point
problem \eqref{e-vonneumann} are a \textit{Nash equilibrium}
\citep{nash1950equilibrium}, though they may not be unique.  Playing a Nash
equilibrium $x^\star$ is \textit{minimax} optimal, you cannot improve on it for
all opponent strategies $y$.  Equation \eqref{e-vonneumann} also yields the
surprising result that there is no advantage to knowing your opponent's strategy
in advance if their strategy is optimal.

\subsection{Learning in repeated matrix games}
\label{sec:learning_games}

Matrix games have a myriad of real-world applications, including economics,
diplomacy, finance, optimization, auctions, and voting systems.  This paper
extends the analysis to the case where the players are also \textit{uncertain}
of the payoff matrix $A$, but can learn about it through their experience.  In each round
$t \in \Nat$ the row player chooses $i_t \in \{1..\ydim\}$ and the column player
chooses $j_t \in \{1..\xdim\}$.  The payment from row player to column player is
given by, \begin{equation} \label{eq:reward}
    r_t = A_{i_t j_t} + \eta_t
\end{equation}
where $\eta_t$ is zero-mean noise, independent and identically distributed from
a known distribution across time. Both players observe the actions of their
opponents and the resulting reward $r_t$, which is referred to as bandit
feedback \cite{lattimore2018bandit}.  We define $\Fc_t = (i_1, j_1, r_1, \ldots,
i_{t-1}, j_{t-1}, r_{t-1})$ to be the sequence of observations available to each
player prior to round $t$, and as shorthand we shall use the notation $\Expect^t
(\cdot) = \Expect(~\cdot \mid \Fc_t)$.  Two aspects of this problem distinguish
it from other setups considered in the literature \citep{cesa2006prediction,
blum2007learning, rakhlin2013optimization}.  Firstly, the players receive the
actions of their opponents as observations, and secondly, the players receive
noisy bandit feedback of the payoff.

We will perform our analysis from the perspective of a single player who does not control the actions of the opponent.
Without loss of generality, we assume control of the \textit{column} player and define a learning algorithm ${\rm alg}$ as a measurable mapping from histories $\Fc_t$ to a distribution over actions $x \in \Delta_\xdim$.
In order to assess the quality of an algorithm ${\rm alg}$ we consider the
\textit{regret}, or shortfall in cumulative rewards, relative to the Nash
equilibrium value,
\begin{equation}
  \label{eq:regret}
  \Rc(A, {\rm alg}, T) = \Expect_{\eta, {\rm alg}} \left[ \sum_{t=1}^T V_A^\star - r_t \right].
\end{equation}
This quantity \eqref{eq:regret} depends on the unknown matrix $A$, which is fixed at the start of play and kept the same throughout.
Expectations are taken with respect to the noise added in the payoffs $\eta$ and the learning algorithm ${\rm alg}$.
To assess the quality of learning algorithms designed to work across some \textit{family} of games $A \in \Ac$ we define:
\begin{eqnarray}
  \label{eq:bayes-regret}
  {\rm BayesRegret}(\phi, {\rm alg}, T) = \Expect_{A \sim \phi} \Rc(A, {\rm alg}, T), \\
  \label{eq:minimax-regret}
  {\rm WorstCaseRegret}(\Ac, {\rm alg}, T) = \max_{A \in \Ac} \Rc(A, {\rm alg}, T).
\end{eqnarray}
The two objectives are sometimes called Bayesian (average-case) \eqref{eq:bayes-regret} and frequentist (worst-case) \eqref{eq:minimax-regret}.
Here, $\phi$ is a prior probability measure over $A \in \Ac$ that assigns relative importance to each problem instance.

\subsection{Main results}
\label{sec:main-results}

The main contribution of this paper is to show that agents employing the
`optimism in the face of uncertainty' (OFU) principle enjoy strong bounds
$\tilde{O}(\sqrt{mkT})$ on both Bayesian and frequentist regret.  Perhaps
surprisingly, these bounds apply to clear and simple applications of K-learning
\citep{o2018variational} and Upper confidence bound (UCB) algorithms
\citep{auer2002finite} and without restriction on the opponent's strategy.
Additionally we show that the \textit{stochastically} optimistic algorithm
Thompson sampling cannot generally enjoy sublinear regret in the presence of an
informed opponent \citep{russo2018tutorial}.  This result clarifies an important
distinction between the applications of the OFU-principle that separates
multi-player games from the single-player setting.
Although we present bounds for the bandit feedback case, it is straightforward
to generalize the results to the case where the agent receives full information,
or information about all the entries in the column and/or row selected.

We supplement our analytical results with a series of didactic
experiments designed to unpick the \textit{empirical} scaling of these
algorithms, and highlight the regimes where one approach may outperform the
other.  In short, we find that for random matrix games, optimistic approaches
that leverage knowledge of the matrix structure perform better than the
adversarial Exp3 algorithm. This computational work is far from
definitive, but may help to guide future work in this nascent area of research.

\section{Applications}
\label{sec:applications}
\paragraph{Uncertain games.}
Any two-player zero-sum game where the agent has uncertainty over the outcomes
of the actions and receives partial feedback is amenable to our framework. Such
examples exists in economics, sociology, politics, psychology and others
\citep{myerson2013game}. Stochastic multi-armed bandits are regularly used in
advertising, but if fraudulent clicks from bots are present then this can be
modeled as a game between the agent and the fraudsters \citep{wilbur2009click}.
Another example is intrusion detection wherein an attacker attempts to penetrate
a system while a defender attempts to prevent the attack, and initially the
players do not know the probability of detection for each pair of actions
\citep{bace2000intrusion}.  Similarly two political parties competing in a
series of election can be modeled in this fashion, where the actions correspond
to targeting messages at different groups of voters and the parties start with
uncertainty about how each action will help or hurt their chances of winning an
election \citep{ordeshook1986game}.

\paragraph{Robust bandits.}
In the robust multi-armed bandit problem the reward of each arm is determined
partially by some other outcome which is selected by `nature'
\citep{caro2013robust, kim2016robust}. The outcomes selected by nature are not
necessarily independent across time-periods nor can we assume that the process
selecting the actions is stationary. It is because of these issues that standard
stochastic multi-armed bandit algorithms fail on this problem. To combat this,
the agent may desire a policy that is \emph{robust}, in the minimax sense, to
all possible selections by nature, which is naturally formulated as a game.
Examples of this problem include clinical trials where one or more
characteristics of the patients are not observed until after the treatment has
been administered \citep{villar2015multi}.  Another is resource placement, where
an agent must place a resource, \eg, a server, in a location and respond to
requests as they come in.  The agent wants to minimize the worst-case response
latency, but does not know in advance the average latency between all pairs of
nodes \citep{cvxccv}.  A further example is route planning, wherein an agent
must decide which route to take to reach some goal but does not know in advance
the average times required to traverse each leg and some exogenous variable
influences the travel times, such as road conditions or traffic
\citep{oliveira2017applying}.  Similar problems exist in A/B testing,
advertising, recommender systems, scheduling, and queueing.

\paragraph{Bandits with budget constraints.}
Consider a multi-armed bandit problem where pulling an arm consumes some amount
$c_i$ of each of $i=1,\ldots,\ydim$ available resources.  Each resource has a
total amount available and the total amount consumed before $T$ time-periods
must be less than this total \citep{badanidiyuru2013bandits}.  This situation is
common in practice and arises, for example, in clinical trials when the
inputs to each of the treatments is not identical and each input has a limited
amount available, or in online advertising where the campaigns have total spend
limits. It turns out this problem can be embedded into a repeated zero-sum
two-player matrix game \cite[\S 4]{immorlica2019adversarial}. In this case the
average reward of each action and the average amount of resource consumed by
each action may be initially unknown.

\section{Optimistic exploration in repeated games}
\label{sec:ofu}

In the literature on efficient exploration, the principle of `optimism in the face of uncertainty' (OFU) has driven the majority of studied algorithms.
This approach assigns a bonus to poorly-understood actions to account for the value of exploration.
The remainder of this section outlines several approaches to exploration driven by OFU, and examines the conditions in which each might be effective.
For the most part, our results mirror those of the bandit literature but, in
some cases, the presence of an opponent raise interesting challenges.

\subsection{Upper confidence bound}
\label{sec:ucb}

Upper confidence bound (UCB) algorithms construct high-probability upper bounds on the value of each
possible action, then (generally) act greedily with respect to those bounds \citep{lai1987adaptive,
auer2002finite}. Carefully controlling how the bounds change over time yield
algorithms that achieve low regret \citep{bubeck2012regret, lattimore2018bandit}.
This is a form of \emph{deterministic} optimism, and it will turn out that
in matrix games this determinism is crucial to prevent exploitation by
the opponent. Before we develop the algorithm, we require the following assumption~\ref{ass-ucb}.
\begin{assumption}
\label{ass-ucb}
The noise process $\eta_t$, $t\in\nats$ is $1$-sub-Gaussian and the
payoff matrix satisfies $A \in [0,1]^{\ydim \times \xdim}$.
\end{assumption}
Under this assumption we can use the Chernoff inequality to provide an
upper bound on each $A_{ij}$ for all $t$ that holds with probability at least
$1-\delta$:
\begin{equation}
\label{e-ucb-ub}
A_{ij} \leq \bar A^t_{ij} + \sqrt{2\log(1/\delta) / (1 \vee n^t_{ij})},
\end{equation}
where $\bar A^t_{ij}$ is the empirical mean of the samples from $A_{ij}$,
$n^t_{ij}$ is the number of times that row $i$ and column $j$ has been chosen by
the players up to (but not including) round $t$, and we have used the notation
$(1 \vee \cdot) = \max(1, \cdot)$. Since we do not control the opponent
we cannot try every possible action once, so
we define the empirical mean $\bar A_{ij}$ to be zero whenever $n^t_{ij} = 0$
and we shall choose $\delta$ such that $\sqrt{2\log(1/\delta)} \geq 1$, which
provides an upper bound on $A_{ij}$ whenever $n^t_{ij} =0$ by assumption that $A
\in [0,1]^{\ydim \times \xdim}$.  This motivates the UCB algorithm presented in
algorithm~\ref{alg-ucb}. The following theorem yields a worst-case regret bound.

\begin{algorithm}[t]
\caption{UCB for matrix games}
\begin{algorithmic}
\FOR{round $t=1,2,\ldots,T$}
\STATE compute $\tilde A^t_{ij} = \bar A^t_{ij} + \sqrt{2\log(2T^2 \ydim \xdim) / (1 \vee n^t_{ij}})$
\STATE use policy $\displaystyle x \in \argmax_{x \in \Delta_\xdim}\min_{y \in \Delta_\ydim} y^T \tilde A^t x$.
\ENDFOR
\end{algorithmic}
\label{alg-ucb}
\end{algorithm}

\begin{theorem}
Let assumption~\ref{ass-ucb} hold with $T \geq \ydim \xdim \geq 2$ and
$\delta = 1/(2T^2 \ydim \xdim)$.
Then, the regret of Algorithm~\ref{alg-ucb} is bounded
\begin{align*}
&{\rm WorstCaseRegret}(\Ac, {\rm UCB}, T)\\
&\qquad\leq 1 + 2 \sqrt{\ydim \xdim T \log \left(2 \ydim \xdim T^2\right)}\\
&\qquad=\tilde O(\sqrt{\ydim \xdim T}).
\end{align*}
\end{theorem}
\begin{proof}
Let $E_t$ be the event that there exists a pair $i, j$ such that
$(\tilde A_t)_{ij} < A_{ij}$.
By definition, $E_t \in \Fc_t$. Consider for a moment that $E_t$ does not hold and let
\begin{align*}
\tilde y_t = \argmin_{y \in \Delta_\ydim} y^\top \tilde A_t x_t
\end{align*}
be the best-response to the player's in round $t$.
Since $E_t$ does not hold, the upper confidence matrix over-estimates the true matrix and hence
$V^\star_{\tilde A_t} \geq V^\star$. Then the per-round regret satisfies
\begin{align*}
V^\star_A - \Expect^t[y_t^\top A x_t] &\leq \Expect^t\left[V^\star_{\tilde A_t} - y_t^\top A x_t\right] \\
&= \Expect^t\left[\tilde y_t^\top \tilde A_t x_t - y_t^\top A x_t\right] \\
&\leq \Expect^t\left[y_t^\top (\tilde A_t - A) x_t\right] \\
&= \Expect^t\sqrt{\frac{2}{1\vee n^t_{i_t j_t}} \log\left(\frac{1}{\delta}\right)}\,,
\end{align*}
where the first inequality follows from optimism and the second since $\tilde y_t$ is the best-response to $x_t$ for matrix $\tilde A_t$.
Next, by the definition of the regret,
\begin{align*}
\Rc(T)
&= \Expect\left[\sum_{t=1}^T V^\star_A -  \Expect^t\left[y_t^\top A x_t\right]\right] \\
&\leq \underbrace{\Expect\left[\sum_{t=1}^T \sqrt{\frac{2 }{1 \vee n^t_{i_t j_t}} \log\left(\frac{1}{\delta}\right)}\right]}_{\textrm{(A)}} + \underbrace{T \Prob\left(\cup_{t=1}^T E_t\right)}_{\textrm{(B)}}\,.
\end{align*}
The second term is bounded naively by $\textrm{(B)} \leq 2T^2\ydim\xdim\delta \leq 1$.
The first term is bounded by
\begin{align*}
\textrm{(A)}
&\leq \sum_{i,j}\Expect\sum_{t=1 : i_t = i, j_t = j}^T \sqrt{\frac{2}{1\vee n^t_{ij}} \log\left(\frac{1}{\delta}\right)}\\
&\leq \sum_{i,j}\Expect\sqrt{4 n^T_{ij} \log\left(\frac{1}{\delta}\right)} \\
&\leq \sqrt{4 \ydim\xdim T \log\left(\frac{1}{\delta}\right)}\,,
\end{align*}
where the final inequality follows from Cauchy--Schwarz. Note that
the inner sum on the first line is from $t=1$ to $T$ where $i_t=i$ and $j_t =
j$, \ie, summing up all the times where action $i,j$ was selected, and the outer
sum is over indices $i,j$.
\end{proof}

\subsection{Thompson sampling}
\label{sec:thompson}

Thompson sampling (TS) is a well-known Bayesian exploration strategy that at
each time period samples an environment according to the posterior probability
over possible environments, then acts greedily with respect to that sample
\citep{thompson1933likelihood,russo2018tutorial,o2017uncertainty}. For matrix
games the Thompson sampling algorithm is described in
algorithm~\ref{alg-ts}.  The performance of UCB algorithms depend strongly on
the confidence sets used to select the action. By contrast it can be shown in
single-player settings that \emph{any} sequence of confidence sets can be used
to bound the Bayesian regret of Thompson sampling \citep{russo2014learning}.  In
this way TS benefits from the best choice of confidence bounds, without
explicitly having to know the best sequence of bounds in advance.  With this in
mind, one might expect a Bayesian regret bound for TS of a similar order to the
bound we just derived for UCB.  In this section we show that, in contrast to
UCB,  we can construct games and opponents that force Thompson sampling to
suffer linear regret.
\begin{algorithm}[t]
\caption{Thompson sampling for matrix games}
\begin{algorithmic}
\FOR{round $t=1,2,\ldots,$}
\STATE sample $\tilde A^t \sim \phi \mid \Fc_t$
\STATE use policy $\displaystyle x \in \argmax_{x \in \Delta_\xdim}\min_{y \in \Delta_\ydim} y^T \tilde A^t x$,
\ENDFOR
\end{algorithmic}
\label{alg-ts}
\end{algorithm}

Take the following $2 \times 2$ game
\begin{equation}
\label{e-2x2game}
  \begin{array}{lr}
\begin{bmatrix}
r & 0 \\ 0 & -1
\end{bmatrix},
    &
r = \left\{\begin{array}{ll}
1 & \mathrm{w.p.}\ 1/2\\
-1 & \mathrm{w.p.}\ 1/2.
\end{array}
\right.
  \end{array}
\end{equation}
Consider the case where the true value of $r=1$, and the TS agent is competing
against an agent that knows the value of $r$ and is simply playing the Nash
equilibrium of $(0, 1)$. The TS agent using algorithm~\ref{alg-ts} will sample
its actions from policy $x = (1, 0)$ with probability $1/2$ and policy $x =
(1/2, 1/2)$ with probability $1/2$.  However, since the other agent is playing
the Nash, the uncertainty about the value of $r$ will never be resolved, and so
the TS agent will have the same behaviour forever. Every time it selects the
second column it incurs a regret of $1$, which happens with probability $1/4$
every time period, thereby yielding linear regret.  This counter-example shows
that Thompson sampling cannot enjoy sub-linear regret against all opponents,
however it does not rule out such bounds in more benign cases, such as self-play
with identical information.

The crucial distinction between Thompson sampling and UCB is the use of
\emph{stochastic}, rather than deterministic, optimism.  This
stochasticity means that sometimes the TS agent is actually pessimistic about
the true state of the world, and in those rounds the agent can be exploited by
an informed opponent.  In the single-player case it can be shown that Thompson
sampling can only suffer high regret in any given round if it is also gaining
information about the optimal action \citep{russo2016information}.
However, in the case with an opponent it is clear that Thompson sampling can
suffer high regret \emph{without} gaining new information. It is in these cases
that TS suffers linear regret, which we shall confirm empirically in the
numerical experiments.

\subsection{Optimistic posterior estimates via K-learning}
\label{sec:k-learn}

K-learning is a Bayesian exploration algorithm originally developed for Markov
Decisions processes in which the agent computes the value of states and actions
using a risk-seeking exponential utility function \citep{o2018variational}. Since the
resulting `K-values' (Knowledge values) are optimistic for the expected values
under the posterior, K-learning can be viewed as employing the OFU principle.
However, it also can be interpreted as a variational approximation to Thompson
sampling \citep{o2020making} which incorporates \emph{deterministic} optimism
while maintaining many of the benefits of Thompson sampling over UCB style
approaches \citep{osb2017optimistic, kaufmann2012bayesian}. Like UCB, the
deterministic optimism is central in the development of a regret bound.
First, let $a_j$ denote the $j$th column of
$A$, which is a random variable with conditional cumulant generating function
$K^t_{a_j}: \reals^m \rightarrow \reals$, defined as
\begin{equation}
\label{e-cgf}
K^t_{a_j}(y) =  \log \Expect^t \exp(a_j^\top y),
\end{equation}
and note that this is the cumulant generating function of $a_j$ under the
\emph{posterior}, conditioned on all the history of observations so far in
$\Fc_t$.
With this in place we present K-learning as algorithm~\ref{alg-klearn}.
The optimization problem in algorithm~\ref{alg-klearn} is convex
and can be expressed as an exponential cone program, for
which efficient algorithms exist \citep{ocpb:16, serrano2015algorithms, ecos}.
We have the following Bayesian regret bound for K-learning.
\begin{algorithm}[t]
\caption{K-learning for matrix games}
\begin{algorithmic}
\FOR{round $t=1,2,\ldots,$}
\STATE $\displaystyle (y_t^\star, \tau_t^\star) \in \argmin_{y \in \Delta_\ydim, \tau \geq 0} \tau \lse{\xdim} K^t_{a_j}(y/\tau)$
\STATE use policy
$x_t^\star \propto \exp K^t_{a_j}(y_t^\star / \tau_t^\star)$
\ENDFOR
\end{algorithmic}
\label{alg-klearn}
\end{algorithm}

\begin{theorem}
Under assumption~\ref{ass-ucb} the K-learning algorithm~\ref{alg-klearn}
satisfies the following Bayesian regret bound
\begin{align*}
  {\rm BayesRegret}(\phi, {\rm Klearn}, T) &\leq 2 \sqrt{\ydim \xdim T \log
\xdim (1+ \log T)}\\
&= \tilde O(\sqrt{\ydim \xdim T}).
\end{align*}
\end{theorem}
\begin{proof}
Using the tower property of expectation we can bound the Bayes regret as
\begin{align}
\begin{split}
\label{e-br1}
&{\rm BayesRegret}(\phi, {\rm Klearn}, T)\\
&\quad= \Expect \sum_{t=1}^T \Expect^t( V_A^\star  - r_t)\\
&\quad= \Expect \sum_{t=1}^T\Expect^t(\min_{y \in \Delta_\ydim} \max_{x \in \Delta_\xdim} y^T A x) - \Expect^t(r_t)\\
  &\quad\leq \Expect \sum_{t=1}^T \min_{y \in \Delta_\ydim}  \Expect^t \max_{x \in \Delta_\xdim} y^T A x
  - y_t^T (\Expect^t A) x_t,
\end{split}
\end{align}
via Jensen's inequality, and the fact that the policies $x_t$ and $y_t$ are
adapted to the filtration $(\sigma(\Fc_t), t\in \nats)$.  Now we shall develop
an upper bound for the expected value of the max.
For any $\tau > 0$,
\begin{align*}
  \Expect^t \max_{x \in \Delta_\xdim} y^T A x &= \Expect^t \max_j a_j^\top y\\
  &\leq \tau \log \Expect^t \exp \max_j a_j^\top y /\tau\\
  &= \tau \log \Expect^t \max_j \exp a_j^\top y /\tau\\
  &\leq \tau \lse{\xdim} K^t_{a_j}(y/\tau),
\end{align*}
where we used Jensen's inequality and the fact that the sum of positive numbers
is greater than the max, and $K^t_{a_j}$ is the cumulant generating function
\eqref{e-cgf}.
We denote by $\Lc^t :\reals^\xdim \times \reals^\ydim \times \reals_+ \rightarrow \reals$, $t=1, \ldots, T$, the Lagrangian
\begin{equation}
\label{e-kl-lagrangian}
  \Lc^t(x, y, \tau) =  \sum_{j=1}^\xdim x_j \tau K^t_{a_j}(y/\tau) + \tau H(x),
\end{equation}
where $H(x) = -\sum_{j=1}^\xdim x_j \log(x_j)$ is the entropy of the agent
policy, and it is straightforward to show that
\[
\tau \lse{\xdim} K^t_{a_j}(y /\tau) = \max_{x \in \Delta_\xdim}\Lc^t(x, y, \tau).
\]
and the $x$ that achieves the maximum is given by
\[
x^\star \propto \exp K^t_{a_j}(y / \tau).
\]

We can bound the first term in the last line of \eqref{e-br1} using
\begin{align*}
   \min_{y \in \Delta_\ydim} \Expect^t \max_{x \in \Delta_\xdim} y^T A x &\leq \min_{y \in \Delta_\ydim,\tau\geq0}\max_{x \in \Delta_\xdim} \Lc^t(x, y, \tau)\\
    &= \Lc^t(x_t^\star, y_t^\star, \tau_t^\star).
\end{align*}
For fixed $x \in \reals^\xdim$ the Lagrangian is jointly convex in $y \in
\reals^\ydim$ and $\tau > 0$, since cumulant generating functions are always
convex and $\tau K^t_{a_j}(y / \tau)$ is the perspective of $K^t_{a_j}$, which
preserves convexity. On the other hand, for fixed $y$ and $\tau \geq 0$ the
Lagrangian is concave in $x$, since entropy is concave \citep{o2016pgq}.
Therefore the Lagrangian is convex-concave jointly in $(y, \tau), x$, which implies
that $\Lc^t(x_t^\star, y_t^\star, \tau_t^\star) \leq \Lc^t(x_t^\star, y, \tau)$
for any feasible $y \in \Delta_\ydim, \tau \geq 0$, due to the saddle point
property.

From this we can bound the Bayes regret incurred in round $t$ from \eqref{e-br1}
\begin{align}
\label{e-kl-t}
\begin{split}
  \Expect^t (V_A^\star - r_t) &\leq \Lc^t(x_t^\star, y_t^\star, \tau_t^\star) - y_t^T
  \Expect^t(A) x_t\\
  &\leq \Lc^t(x_t^\star, y_t, \tau_t) - y_t^T\Expect^t(A) x_t,
\end{split}
\end{align}
where $y_t$ is the strategy played by the opponent, and $\tau_t \geq 0$ is a
free parameter. Assumption~\ref{ass-ucb} implies that the posterior of $a_j$,
$j=1, \ldots, \xdim$, is $1$-sub-Gaussian and concentrates as
\begin{equation}
\label{e-asssubg}
  \tau K^t_{a_j}(y / \tau) \leq (\Expect^t a_j)^T y + \sum_{i=1}^m \frac{y_i^2}{2 \tau
  (1 \vee n^t_{ij})}.
\end{equation}
Now all that remains is to bound the sum over time
using equation (\ref{e-kl-t}) and equation~\eqref{e-asssubg}
\begin{align*}
&{\rm BayesRegret}(\phi, {\rm Klearn}, T) \\
&\quad\leq \Expect \sum_{t=1}^T\left[\sum_{i,j}
  \frac{x^\star_{tj} y_{ti}^2}{2 \tau_t
  (1 \vee n^t_{ij})} + \tau_t H(x_t^\star) \right]\\
&\quad\leq\Expect \sum_{t=1}^T \left[\sum_{i,j} \frac{x^\star_{tj} y_{ti}}{2 \tau_t
  (1 \vee n^t_{ij})} + \tau_t H(x_t^\star) \right]\\
&\quad\leq  \ydim \xdim (1 + \log T) /2 \tau_T + \log \xdim \sum_{t=1}^T \tau_t \\
&\quad\leq 2 \sqrt{\ydim \xdim T \log \xdim (1+ \log T)}.
\end{align*}
where the third inequality follows from a pigeonhole argument which we
present as lemma~\ref{l-pigeonhole} below, and the last inequality
sets free parameter $\tau_t = \sqrt{\ydim \xdim (1 + \log T)
/ (4 t \log \xdim)}$.
\end{proof}
\begin{lemma}
\label{l-pigeonhole}
Consider a process that at each time $t$ selects a single index $a_t$
from $\{1, \ldots, q\}$ with probability $p^t_{a_t}$. Let $n^t_{i}$ denote the
count of the number of times index $i$ has been selected before time $t$, and
assume that $T \geq 1$. Then
\[
  \sum_{t=1}^T \sum_{i=1}^{q} p^t_{i} / (1 \vee n^t_{i}) \leq q (1 + \log T).
\]
\end{lemma}
\begin{proof}
This follows from a straightforward application of the pigeonhole principle,
\begin{align*}
  \sum_{t=1}^T \sum_{i=1}^{q} p_i^t / (1 \vee n_i^t)
  &= \sum_{t=1}^T \Expect_{a_t \sim p^t} (1 \vee n^t_{a_t})^{-1}  \\
&= \Expect_{a_0 \sim p^0, \ldots,
  a_T \sim p^t} \sum_{t=1}^T (1 \vee n^t_{a_t})^{-1} \\
&= \Expect_{a_0 \sim p^0, \ldots,
a_T \sim p^t} \sum_{i=1}^{q} \sum_{t=1}^{1 \vee n_i^T} 1/t \\
&\leq \sum_{i=1}^{q} \sum_{t=1}^{T} 1/t  \\
  &\leq q (1 + \log T),
\end{align*}
since $n_i^t$ is the count \emph{before} time $t$ we have $1 \vee n_i^T \leq T$
for each $i$ and where the last inequality follows since $\sum_{t=1}^{T} 1/t
\leq 1 + \int_{t=1}^{T} 1/t = 1 + \log T$.
\end{proof}

Now we return to the simple $2\times 2$ problem with payoff matrix in equation
\eqref{e-2x2game}.  Recall that Thompson sampling will incur linear regret in
this setting since it will select the second column with probability $1/4$ each
round.  By contrast, a quick calculation tells us that K-learning in this
situation will always play the strategy $(1,0)$, thereby incurring zero regret,
and will play this forever even though the uncertainty about the value of $r$ is
never resolved. This is demonstrated in Figure~\ref{fig-2x2game}.

\begin{figure}
  \begin{center}
    \includegraphics[width=0.9\linewidth]{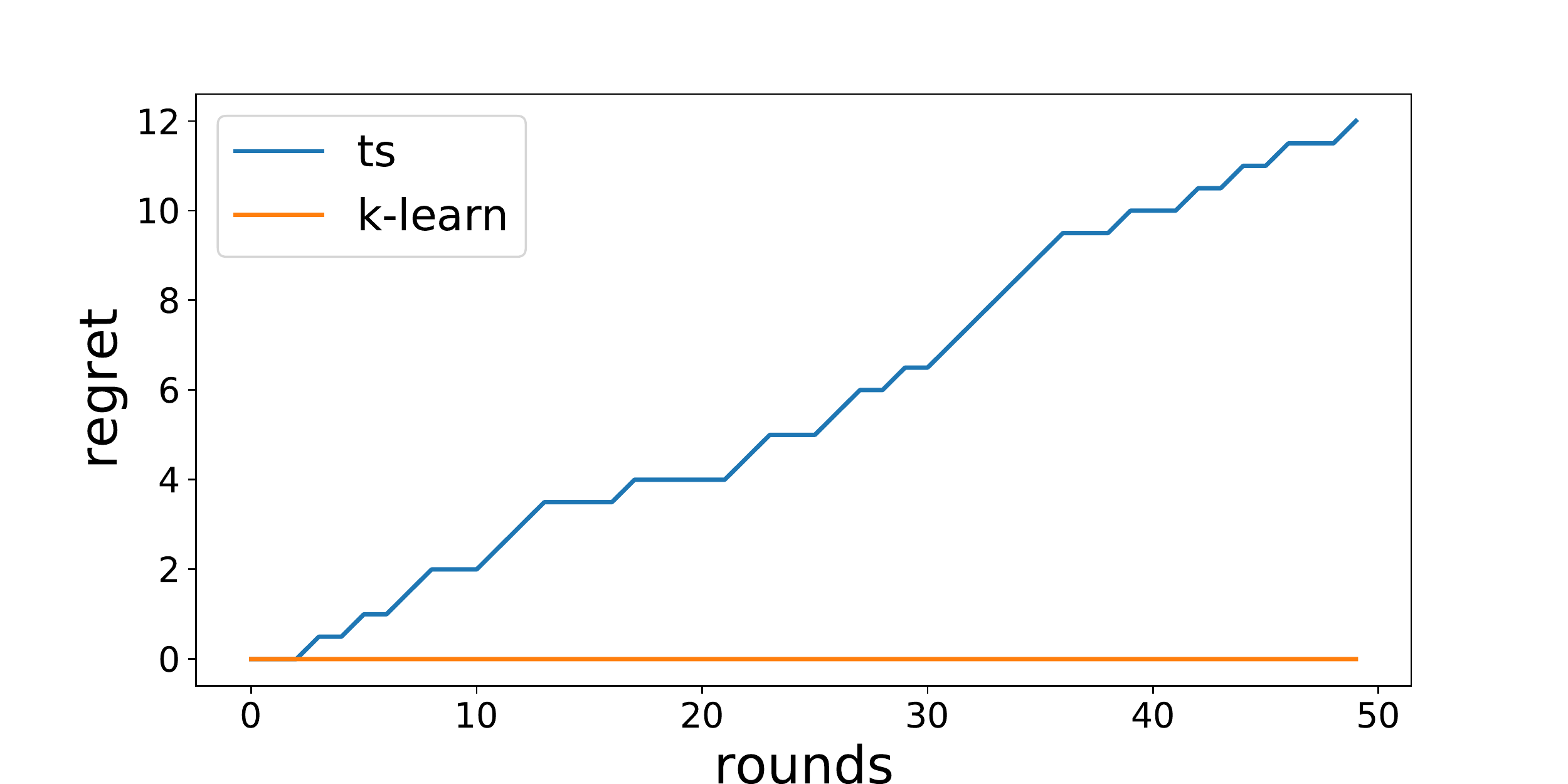}
  \end{center}
  \caption{Regret for game matrix Eq.~(\ref{e-2x2game}).}
  \label{fig-2x2game}
\end{figure}

\section{Adversarial bandit algorithms}

In the adversarial bandit framework, an adversary and learner interact sequentially over $T$ rounds. In each round $t$,
the learner chooses a distribution $x_t \in \Delta_k$ and the adversary simultaneously chooses a loss vector $\ell_t \in [0,1]^k$.
Any algorithm designed for adversarial bandits can be used in our setting by choosing $\ell_{ti} = -A_{j_ti}$.
The usual definition of the regret in this notation is
\begin{align*}
\underbrace{\Expect\left[\max_i \sum_{t=1}^T A_{j_t i} - \sum_{t=1}^T r_t\right]}_{\text{adversarial regret}}
  \geq \Rc(A, {\rm alg}, T)\,.
\end{align*}
Hence, an algorithm with small adversarial regret automatically enjoys small regret relative to the Nash strategy \citep{hannan1957approximation}.
There are now many algorithms for adversarial bandits,
the most well-known being Exp3 \citep{auer1995gambling}.
The basic algorithm uses
importance-weighting to estimate the rewards for each action and samples from a
carefully tuned exponential weights distribution. Let $\hat r_{ti}$ be the importance-weighted
estimate of the reward of action $i$ in round $t$:
\begin{align*}
\hat r_{ti} = \frac{r_t\ones(i_t = i)}{x_{ti}}\,,
\end{align*}
where the distribution of the player $x_t$ is given by
\begin{align*}
x_{ti} = \frac{\gamma_t}{k} + (1 - \gamma_t) \frac{\exp\left(\rho_t \sum_{s=1}^{t-1} \hat r_{si}\right)}{\sum_{j=1}^k \exp\left(\rho_t \sum_{s=1}^{t-1} \hat r_{sj}\right)}\,.
\end{align*}
When $\rho_t$ and $\gamma_t$ are tuned appropriately, then the regret of Exp3 relative to the best action in hindsight is
\begin{align}
\Expect\left[\max_i \sum_{t=1}^T A_{j_t i} - \sum_{t=1}^T r_t \right] = O\left(\sqrt{kT \log k}\right)\,. \label{eq:exp3-regret}
\end{align}
The reader will notice that this bound is independent of the number of actions
of the opponent, which was not true for UCB or K-learning.  Another strength of
Exp3 and similar algorithms is that the alternative notion of regret means they
can exploit weak opponents.  On the other hand, Exp3 is empirically much worse
than K-learning and UCB. The reason is that Exp3 does not use the structure of the game
and cannot quickly eliminate actions that do not play a strong role in any
plausible Nash equilibrium. Furthermore, in many cases the goal is
to learn the Nash equilibrium (if possible), \ie, to have `solved' the game, not
just to exploit the opponent. For example, since Exp3 does not converge to the
minimax solution in general, it does not solve the \emph{robust} bandit problem
and suffers from high variance of $\Omega(T^2)$ \cite[Ex.~11.6]{lattimore2018bandit}.
Concretely, consider playing rock-paper-scissors against an opponent with fixed
strategy $(0.2, 0.2, 0.6)$. Exp3 against this opponent will converge towards
playing $(1, 0, 0)$. However, a UCB or K-learning agent will learn to play the Nash strategy
$(1/3, 1/3, 1/3)$, and will not be exploitable by any opponent (\ie, they will
be robust), even though they only played against a weak player. If suddenly the
opponent changes then Exp3 will suffer significantly larger losses than the
robust algorithms even though the final regret may not be worse. We
shall demonstrate this phenomenon in the numerical experiments.

There are many adaptations of Exp3. The main threads are (a) using the online
convex optimisation view and modifying the regularizer
\citep{audibert2009minimax,BCL17,CL18}, for example, and (b) modifying the loss
estimates to obtain high probability regret or adaptive bounds
\citep{exp3,KNVM14,Neu15,AHR08}.
None of these algorithms are able to handle the additional knowledge of the
opponent's action and we do not believe any will improve on Exp3 by a
significant margin empirically.  The partial monitoring framework \textit{can}
incorporate knowledge of the opponent's action \citep{Rus99}.  Partial monitoring
is now reasonably well understood theoretically \citep{BFPRS14} and sensible
algorithms exist \citep{LS19pmsimple}. Regrettably, however, even with Bernoulli
rewards, the matrix games studied here can only be modelled by exponentially
large partial monitoring games for which existing algorithms are not practical.
The case of two-player matrix games where the
matrix $A$ is selected adversarially at each timestep was considered in
\citep{cardoso19a}, however, that work assumed control of \emph{both} players, so
is not applicable here.

\section{Numerical experiments}

In this section we present numerical results comparing the performance
of the algorithms we have discussed so far.
In most cases we are interested in measuring the empirical regret on a
particular problem. Since this depends on the opponent we shall report
cumulative \emph{absolute} regret, \ie,
\[
\sum_{t=1}^T |V_A^\star  - y_t^T A x_t |,
\]
for fixed $A$.  This
is meaningful because we primarily focus on two cases: self-play and against a
best-response opponent.  In self-play the algorithm is competing against another
player using the same algorithm with the same information and so the cumulative
absolute regret is a loosely measure of how far the players are from the Nash
equilibrium.
The best-response opponent
knows the exact value of $A$ and the agent's strategy at every round, and so can
compute the action that minimizes the expected payoff.  In this case the regret
the agent suffers is always positive, so the absolute regret is the same as the
usual notion of regret.

When running Exp3 we used the following parameters
\[
\gamma_t = \min(\sqrt{\xdim \log \xdim / t}, 1), \quad
\rho_t = \sqrt{2 \log \xdim / t \xdim}.
\]

\subsection{Rock-paper-scissors}
In the classic children's game rock-paper-scissors, the payoff matrix is given by
\begin{center}
\begin{tabular}{c| c c c}
& R & P & S\\
\hline
R & 0 & 1 & -1\\
P & -1 & 0 & 1\\
S &1 & -1 & 0,
\end{tabular}
\end{center}
which defines a symmetric game
with Nash equilibrium $(1/3, 1/3, 1/3)$ for both players.
When comparing the techniques on this problem we add noise $\eta_t \sim \Nc(0,
1)$ to the payoff, and use prior $\Nc(0, 1)$ for each entry in the matrix for
the Bayesian algorithms.  We ran each experiment for $1000$ rounds and averaged
the results over $100$ seeds.

In Figure~\ref{fig-rps-self-play} we present the self-play results and in
Figure~\ref{fig-rps-br} we show the results against a best-response opponent. In
both cases we plot the absolute regret of each algorithm and the KL-divergence
of the policy produced by each algorithm to the Nash equilibrium policy. In
self-play K-learning and UCB perform well with low regret and relatively quick
convergence towards the Nash. Although Thompson sampling doesn't enjoy a regret
bound against all opponents, it still appears to perform well in self-play.
Exp3, which does not use the matrix structure of the problem, does not converge
to the Nash equilibrium in self-play in this case. This is shown by the linear
absolute regret and the KL-divergence to the Nash saturating at a constant.
Although Exp3 has a regret bound, the two competing instantiations oscillate
around the Nash together, sometimes winning and sometimes losing (on average)
against their opponent. It is clear from this result that Exp3 is not guaranteed
to solve the game and converge to the Nash equilibrium, and so cannot solve the
robust bandit problem in general without further assumptions.  Against the
best-response opponent the major difference is the dramatic decline in
performance for Thompson sampling. It is clear that even in this simple case TS
is easily exploited by an informed opponent and suffers significant losses. In
contrast to self-play, against the best-response opponent Exp3 \emph{will}
converge to the Nash equilibrium, since it satisfies a regret bound and the Nash
is the only strategy that is not exploitable.  This is shown by the (slow)
convergence in KL-divergence between the Exp3 policy and the Nash towards zero.

\begin{figure}
\centering
\begin{subfigure}{0.7\linewidth}
\centering
  \includegraphics[width=\linewidth]{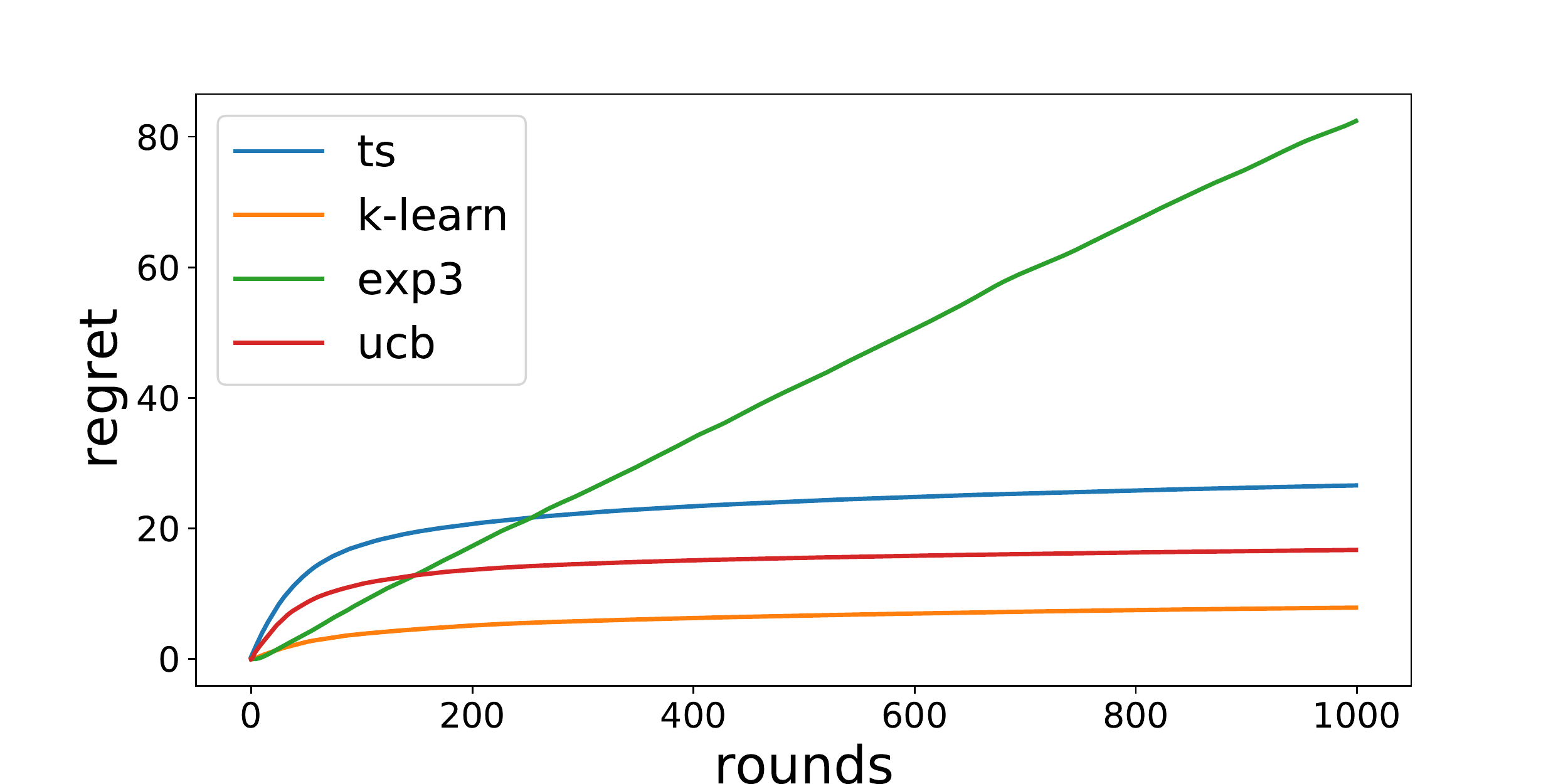}
  \caption{Absolute regret.}
\label{fig-rps}
\end{subfigure}
\begin{subfigure}{0.7\linewidth}
\centering
  \includegraphics[width=\linewidth]{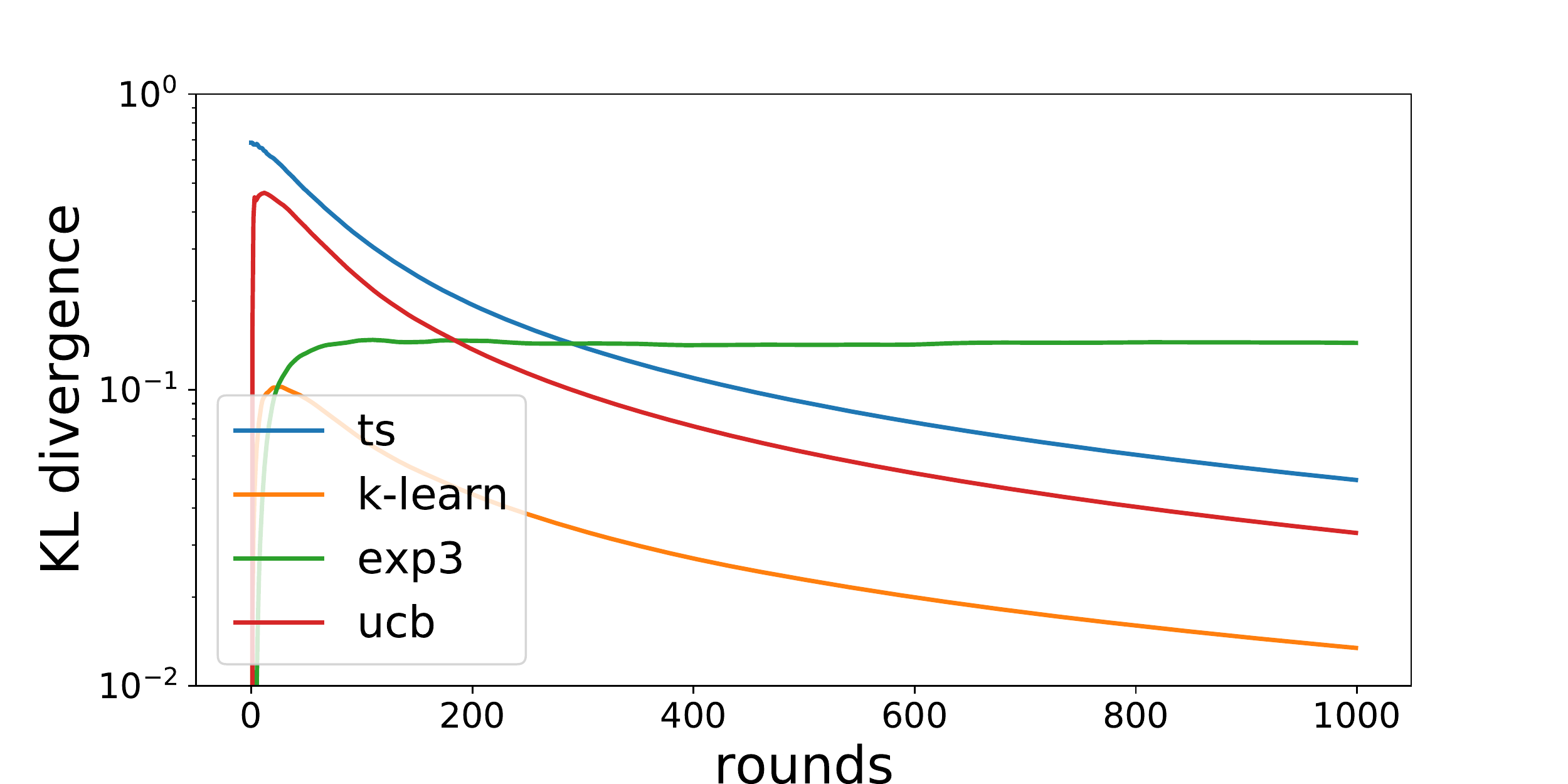}
  \caption{KL divergence to Nash.}
\label{fig-rps-kldiv}
\end{subfigure}
\caption{Rock-paper-scissors self-play.}
\label{fig-rps-self-play}
\end{figure}

In Figure~\ref{fig-rps-compete} we compare the performance of the algorithms
with regret bounds competing against each other with identical information. The
legend displays `alg1 vs alg2' for different choices of alg1 and alg2, and
indicates that alg1 is playing as the maximizer and alg2 is the minimizer.  We
are plotting the regret (not absolute regret) from the point of view of the
maximizer (alg1). If the regret is positive, it means that the minimizer (alg2)
is winning on average.  Since rock-paper-scissors is symmetric there is no
advantage to being one player or the other so this is a fair head-to-head
comparison. From the figure it is immediately obvious that the algorithms that
leverage the matrix structure, K-learning and UCB, are easily beating Exp3 on
average.  Although Exp3 has a regret bound, it requires a long time to learn and
in the meantime it suffers large losses against the optimistic approaches. When
K-learning competes against UCB the algorithms are roughly evenly matched,
however it appears that UCB has a slight advantage in this case.

\begin{figure}
\centering
\begin{subfigure}{0.7\linewidth}
\centering
  \includegraphics[width=\linewidth]{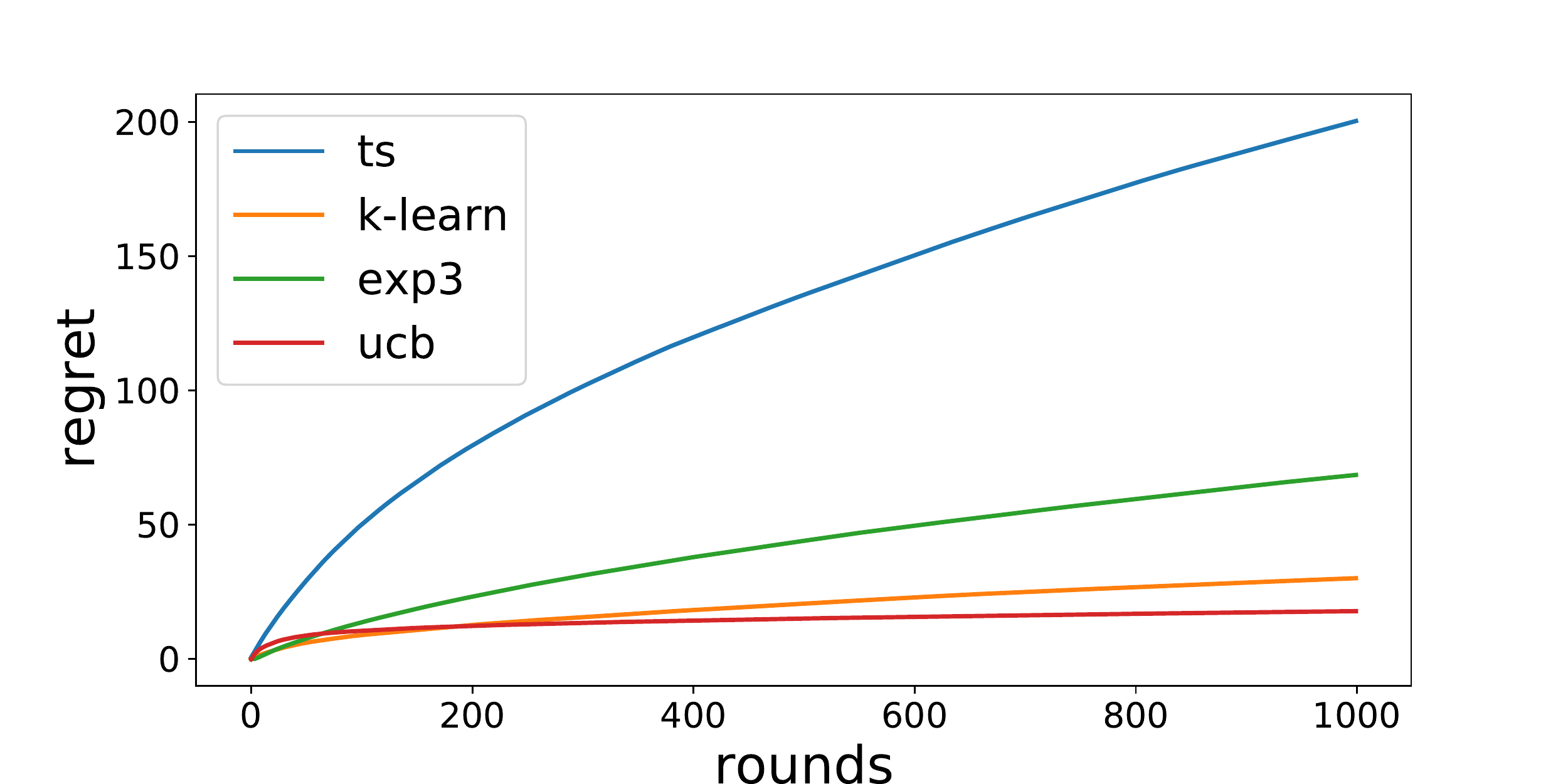}
  \caption{Regret.}
\label{fig-rps-br-regret}
\end{subfigure}
\begin{subfigure}{0.7\linewidth}
\centering
  \includegraphics[width=\linewidth]{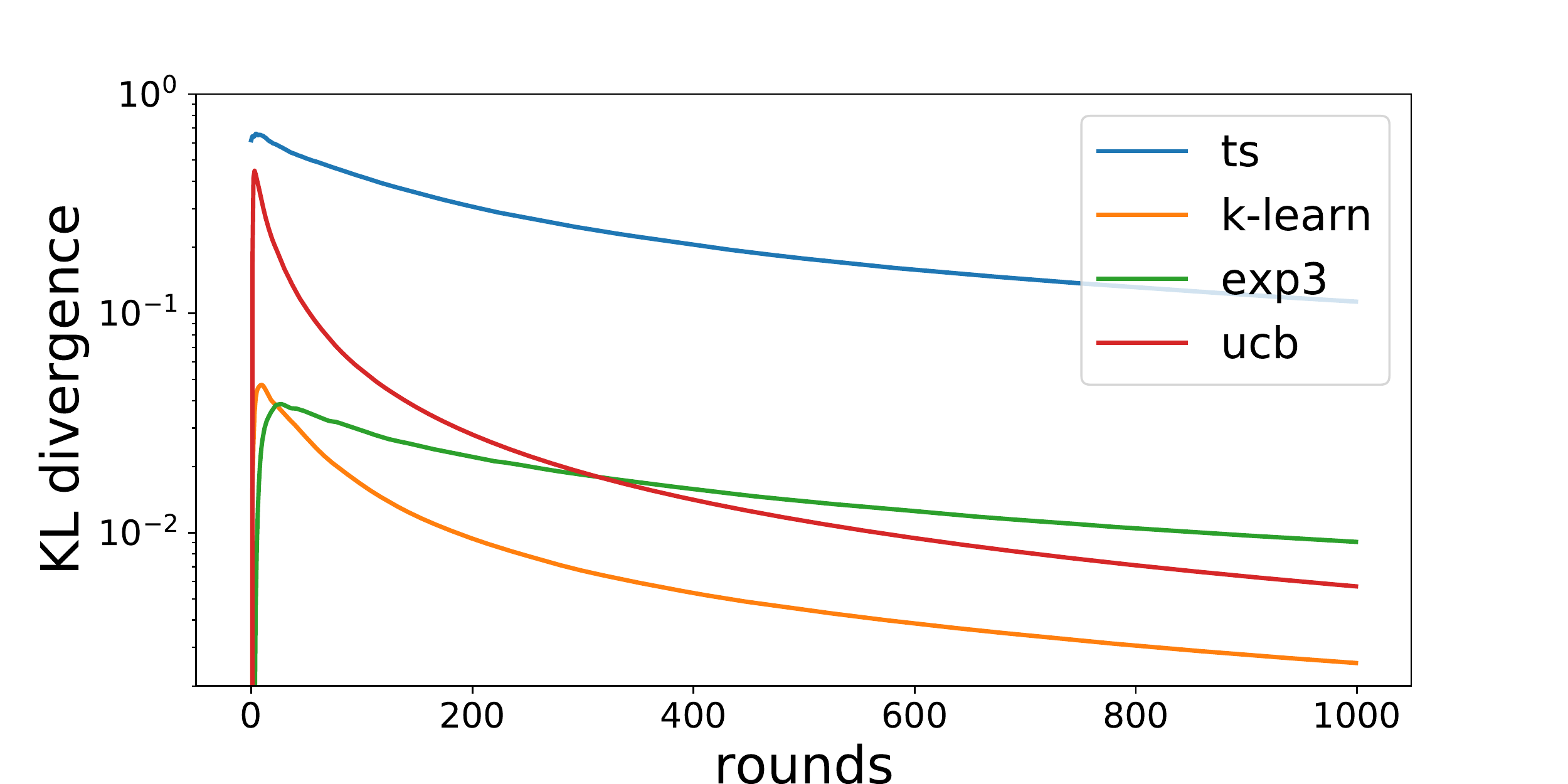}
  \caption{KL divergence to Nash.}
\label{fig-rps-br-kldiv}
\end{subfigure}
\caption{Rock-paper-scissors vs.~best-response.}
\label{fig-rps-br}
\end{figure}

\begin{figure}
\centering
  \includegraphics[width=0.8\linewidth]{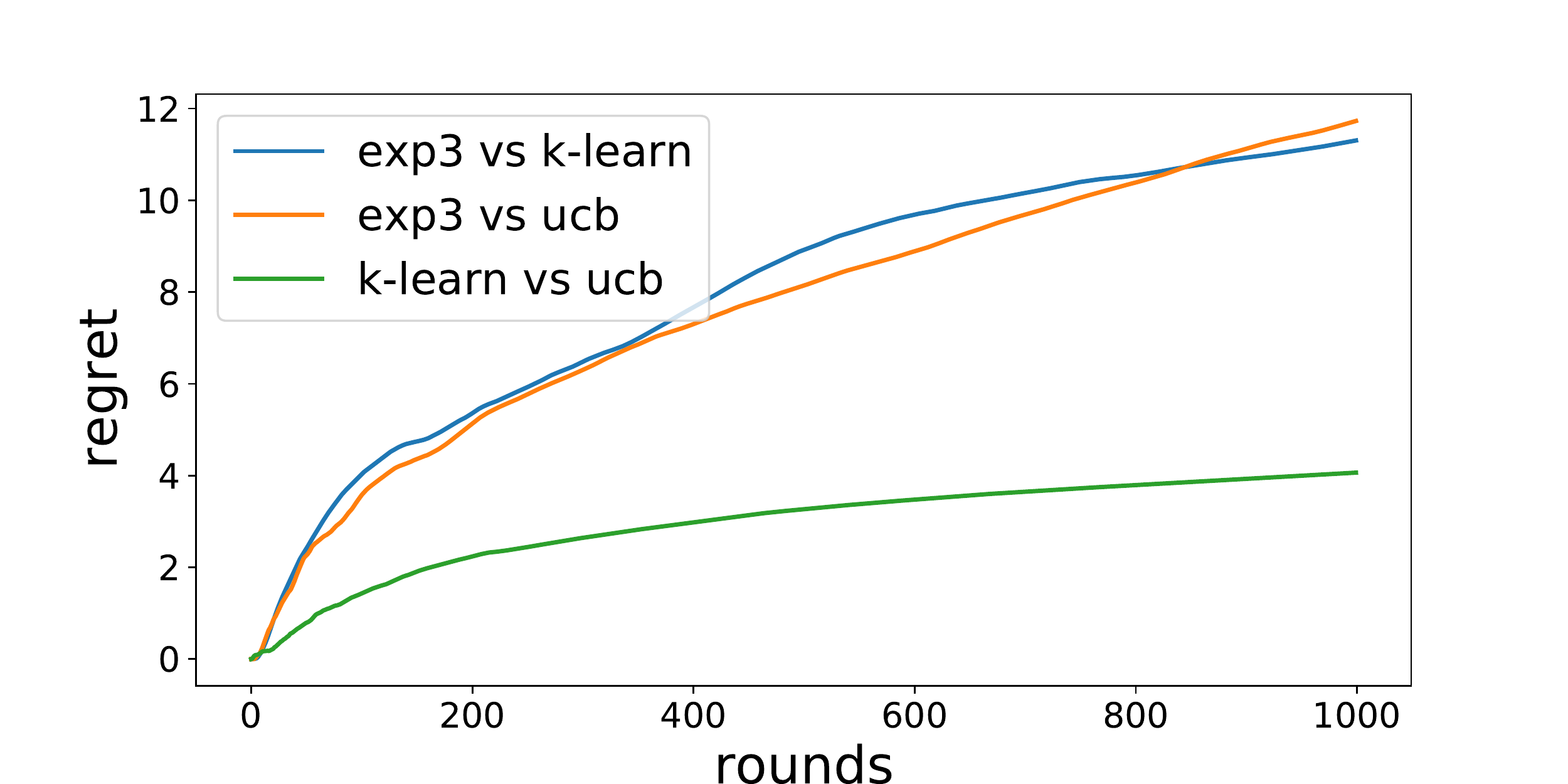}
\caption{Rock-paper-scissors head-to-head regret.}
\label{fig-rps-compete}
\end{figure}

\subsection{Robust bandits}
In the robust bandit problem the rewards the agent receives are partially
determined by outcomes selected by nature, and the agent wants a policy that is
robust to all possible outcomes. This problem can be formulated as a game to
which we can apply the algorithms we have developed.  To test their performance
we generated a random game with $\xdim=10$ agent actions and $\ydim=5$ possible
outcomes for each action, where each entry of $A$ was sampled IID from $\Nc(0.5,
2.0)$. Nature sampled actions from a fixed policy that changed randomly every
$50$ time-steps.  We compare the algorithms presented in this manuscript against
naive UCB and naive Thompson sampling, which treated the problem as though it
was a standard stochastic multi-armed bandit problem. We ran each algorithm for
$1000$ time-steps averaged over $100$ random seeds and we plot the histogram of
the per time-period rewards in Figure~\ref{fig-rob-hists}. In
Table~\ref{t-rob-bandit} we show what proportion of the time each algorithm
suffered a negative reward, as well as the average reward of each approach. It
is clear that the naive approaches suffer from negative rewards more frequently,
\ie, they are not robust to the changing conditions of nature. For example,
K-learning suffers negative rewards almost $16\times$ less frequently than the
naive approaches, which both suffered negative rewards about $15\%$ of the time,
at the expense of slightly lower average reward. We can also see that Exp3 is
not robust since it too suffers significantly more negative rewards than
K-learning and UCB.  Since Exp3 attempts to exploit the nature player, it can
suffer large negative rewards for several periods when nature switches
distribution.  For completeness we include the results of the same problem
against a best-response opponent, summarized in Table~\ref{t-rob-bandit-br}.
Unsurprisingly the naive approaches are trivially exploitable by the BR opponent
and suffer large negative rewards at every time-step. Again, Exp3 suffers
significantly more negative reward than the robust approaches in this case.
K-learning has both the largest average reward and the least percentage of
negative rewards overall, followed by UCB.

\begin{figure}
  \includegraphics[width=1.\linewidth]{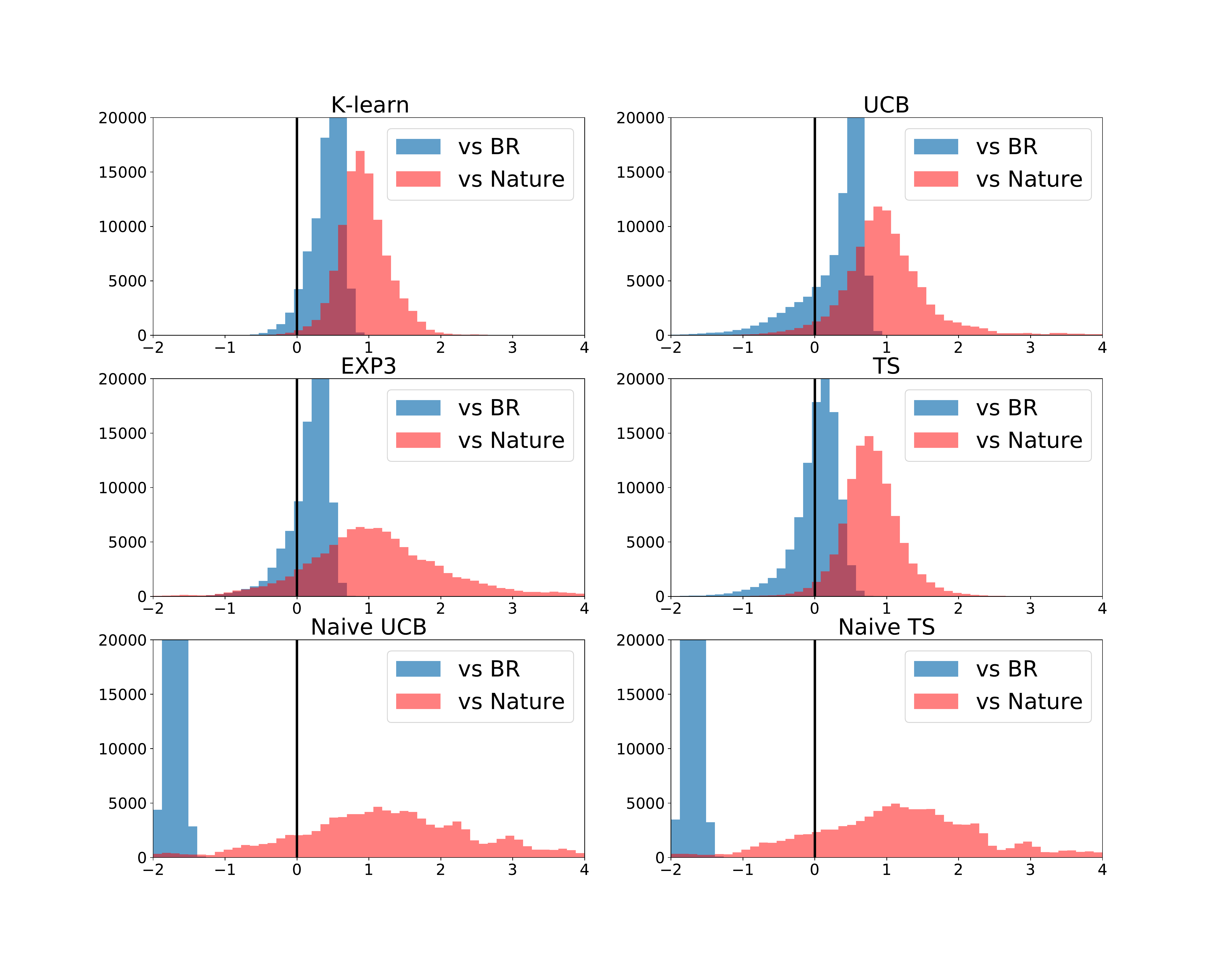}
  \caption{Reward histograms on robust bandit problem.}
\label{fig-rob-hists}
\end{figure}

\begin{table}
\begin{center}
\begin{tabular}{c | c c}
&\% returns $< 0$ & mean return \\
TS & 2.3\% & 0.787 \\
UCB & 3.4\% & 0.995 \\
K-learn & 0.7\% & 0.935 \\
Exp3 & 9.0\% & 1.092 \\
Naive TS & 15.9\% & 1.159 \\
Naive UCB & 14.1\% & 1.269 \\
\end{tabular}
\end{center}
\caption{Robust bandit problem vs Nature.}
\label{t-rob-bandit}
\end{table}

\begin{table}
\begin{center}
\begin{tabular}{c | c c}
&\% returns $< 0$ & mean return \\
TS & 37.4\% & 0.027 \\
UCB & 18.8\% & 0.299 \\
K-learn & 5.1\% & 0.421 \\
Exp3 & 19.5\% & 0.180 \\
Naive TS & 100.0\% & -1.703 \\
Naive UCB & 100.0\% & -1.710 \\
\end{tabular}
\end{center}
\caption{Robust bandit problem vs best-response.}
\label{t-rob-bandit-br}
\end{table}

\section{Conclusion}
The usual analysis of matrix games assumes that both players have perfect
knowledge of the payoffs.  We extended this to the case where the matrix that
specifies the game is initially unknown to the players and must be learned about
from experience, specifically from noisy bandit feedback. We showed that two
previously published algorithms, UCB and K-learning, can be extended to this
case and enjoy a sublinear regret bound, even against informed opponents that
can compute a best-response to their strategies. We also showed a
counter-example that rules out a sublinear regret bound for Thompson sampling
under the same conditions.  This difference between deterministically optimistic
and stochastically optimistic algorithms is a significant departure from the
single-player case. We supported our findings with numerical
experiments that showed a significant advantage of these approaches when
compared to both Thompson sampling and Exp3.

We conclude with a brief discussion about lower bounds. The two optimistic
algorithms we developed in this manuscript have $O(\sqrt{T})$ regret upper
bounds. One might speculate about the existence of matching lower bounds. Since
the matrix game generalizes the stochastic multi-armed bandit problem we know we
cannot improve upon this bound in the worst-case. However, this does not
preclude the existence of an instance-dependent logarithmic regret bound. We conjecture
that such a bound is not possible in general against all opponents, though it
may be possible in more benign cases such as self-play with identical
information. We leave exploring this to future work.

\bibliography{k_learn_games}

\end{document}

%% file: bod_macros.tex
\newcommand{\eg}{{\it e.g.}}
\newcommand{\ie}{{\it i.e.}}

\newcommand{\BA}{\begin{array}}
\newcommand{\EA}{\end{array}}

\newcommand{\BIT}{\begin{itemize}}
\newcommand{\EIT}{\end{itemize}}

\newcommand{\ones}{\mathbf 1}

\newcommand{\reals}{{\mathbb{R}}} %
\newcommand{\nats}{{\mathbb{N}}} %

\newcommand{\Expect}{\mathbb{E}}
\newcommand{\Prob}{\mathbb{P}}

\newcommand{\argmin}{\mathop{\rm argmin}}
\newcommand{\argmax}{\mathop{\rm argmax}}